\documentclass[a4paper,10pt,oneside]{amsart}

\theoremstyle{plain}
\newtheorem{lemma}{Lemma}[section]
\newtheorem{proposition}[lemma]{\textbf{Proposition}}
\newtheorem{theorem}[lemma]{\textbf{Theorem}}
\newtheorem{cor}[lemma]{\textbf{Corollary}}
\theoremstyle{definition}
\newtheorem{definition}[lemma]{\textbf{Definition}}

\newtheorem*{notation}{\textbf{Notation}}
\theoremstyle{remark}
\newtheoremstyle{indenteddefinition}{8pt}{8pt}{\addtolength{\leftskip}{1.4em}}{
-1.4em}{\itshape}{.}{ }{}
\theoremstyle{indenteddefinition}
\newtheorem{remark}[lemma]{Remark}

\newcommand{\R}{\mathbb{R}}
\newcommand{\C}{\mathbb{C}}

\newcommand{\p}{\mathbb{P}}

\newcommand{\SE}{\mathrm{SE}_3}
\newcommand{\SEC}{\mathrm{SE}_{3,\C}}
\newcommand{\SO}{\mathrm{SO}_3}

\renewcommand\epsilon{\varepsilon}

\newcommand{\tth}{\thinspace}
\newcommand{\mcal}[1]{\mathcal{#1}}

\usepackage[english]{babel}
\usepackage[T1]{fontenc}
\usepackage{amsmath,amssymb,amsthm,amsfonts}
\usepackage{latexsym}
\usepackage{mathrsfs}
\usepackage{mathtools}
\usepackage{faktor}
\usepackage{nicefrac}
\usepackage[all]{xy}
\usepackage{changepage}
\usepackage{multirow}
\usepackage{enumerate}
\usepackage{overpic}
\usepackage{graphics}
\usepackage{graphicx}
\usepackage{subfigure}
\usepackage{algorithm}
\usepackage{algpseudocode}

\algrenewcommand\algorithmicwhile{\textbf{While}}
\algrenewcommand\algorithmicfor{\textbf{For}}
\algrenewcommand\algorithmicdo{\textbf{Do}}
\algrenewcommand\algorithmicif{\textbf{If}}
\algrenewcommand\algorithmicthen{\textbf{Then}}
\algrenewcommand\algorithmicelse{\textbf{Else}}
\algrenewcommand\algorithmicend{\textbf{End}}
\algrenewcommand\algorithmicreturn{\textbf{Return}}

\begin{document}

\title{Bond Theory for pentapods and hexapods}

\author[Gallet]{Matteo Gallet}

\address{
Research Institute for Symbolic Computation \\
Johannes Kepler University \\
Altenberger Strasse 69 \\
4040 Linz, Austria.}

\email{mgallet@risc.jku.at}

\author[Nawratil]{Georg Nawratil}

\address{
Institute of Discrete Mathematics and Geometry\\  
Vienna University of Technology \\
Wiedner Hauptstrasse 8-10/104 \\
1040 Vienna, Austria.}

\email{nawratil@geometrie.tuwien.ac.at}

\author[Schicho]{Josef Schicho}

\address{
Research Institute for Symbolic Computation \\
Johannes Kepler University \\
Altenberger Strasse 69 \\
4040 Linz, Austria.}

\email{josef.schicho@risc.jku.at}


\begin{abstract}
This paper deals with the old and classical problem of determining necessary
conditions for the overconstrained mobility of some mechanical device. In
particular, we show that the mobility of pentapods/hexapods implies either
a collinearity condition on the anchor points, or a geometric condition on the
normal projections of base and platform points. The method is based on a
specific compactification of the group of direct isometries of $\R^3$.
\end{abstract}

\maketitle


\section{Introduction}

The objects we will focus on in this paper are the so--called \emph{$n$--pods}. 
For $n=5$ they are referred as pentapods and for $n = 6$ as hexapods, which are also 
known as \emph{Stewart Gough platforms}. 
As described in \cite{Nawratil2014}, the geometry of this kind of mechanical manipulators is defined by the coordinates of the $n$ platform anchor points $p_i = (a_i, b_i, c_i) \in \R^3$ and of the $n$ base anchor points $P_i = (A_i, B_i, C_i) \in \R^3$ in one of their possible configurations. All pairs of points $(p_i, P_i)$ are connected by a rigid body, called \emph{leg}, so that for all possible configurations the distance $d_i = \left\| p_i - P_i \right\|$ is preserved. 
\begin{figure}[!ht]
	\begin{center}  
		\begin{overpic}[width=0.35\textwidth]{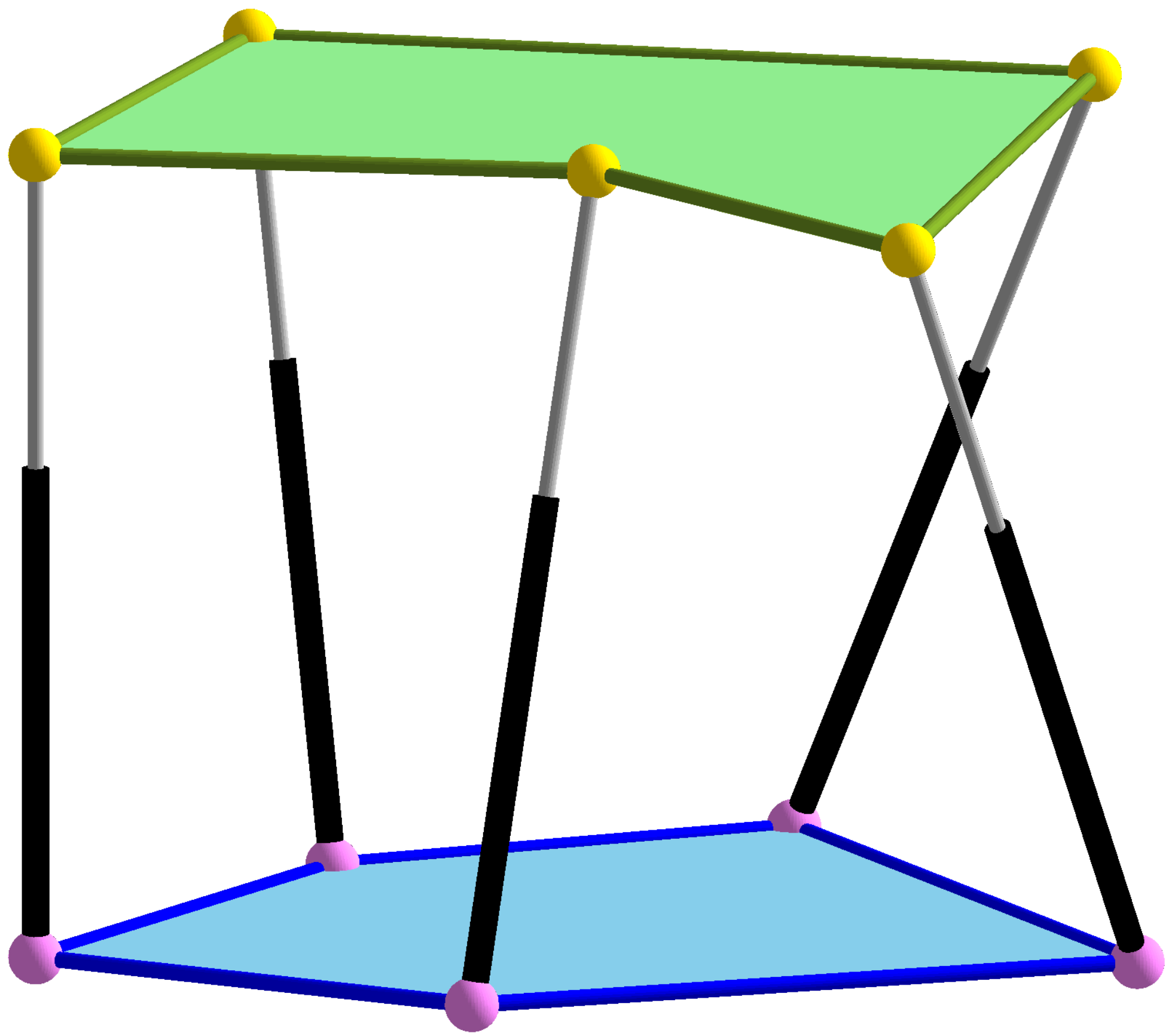}
			\begin{small}
				\put(-8,4){$P_1$}
				\put(32,-5){$P_2$}
				\put(99.5,4){$P_3$}
				\put(58,21){$P_4$}
				\put(16,16){$P_5$}
				\put(-6,72){$p_1$}
				\put(41,68){$p_2$}
				\put(69.5,61){$p_3$}
				\put(96.5,81){$p_4$}
				\put(12,87){$p_5$}
			\end{small} 
		\end{overpic} 
	\end{center}
	\caption{Sketch of an $n$-pod for $n = 5$, a pentapod.}
\end{figure}
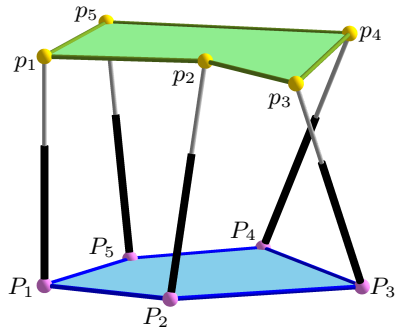   
\begin{notation}
	We will think of an $n$--pod $\Pi$ as a triple
	\[ \Pi = \Big( (p_1,\dots,p_n),(P_1,\dots,P_n),(d_1,\dots,d_n) \Big) \]
	where $p_i$, $P_i$ and $d_i$ are defined as above.
\end{notation}
We are interested in describing the \emph{self--motions} of a given $n$--pod
$\Pi$, namely which direct isometries $\sigma$ of $\R^3$ satisfy the condition
(which is called the \emph{spherical condition})
\begin{equation}
\label{equation:spherical_condition}
	\left\| \sigma(p_i) - P_i \right\| \; = \; d_i \quad \quad \textrm{for all }
i \in \{1, \ldots, n\}
\end{equation}
In particular we want to understand what is the \emph{dimension} of the set of
these isometries, namely the \emph{mobility} of $\Pi$ (we will make these
concepts precise at the beginning of Section \ref{bond_geometric}), and what
conditions we have to impose on the base and platform points to reach a
prescribed mobility. 

In order to do so we first study the group of direct isometries of $\R^3$, and
in particular we embed and compactify it in a projective space in a way
specifically tuned for $n$--pods. The key idea behind
all of this is that, if we introduce suitable coordinates in the space of
direct isometries and we consider the condition given by Equation
(\ref{equation:spherical_condition}) then the latter becomes
linear in these coordinates. These coordinates provide the desired
compactification, and Section \ref{compactification} is devoted to the study of
some of its geometric properties, which play an important role in the proofs of
the results of the subsequent section. In particular, we focus our attention on
the natural action of direct isometries on this compactification and on its
boundary. In Section \ref{bond_geometric} we use this information to establish
some results which can be framed in the so--called \emph{bond theory} (see
\cite{HegeduesSchichoSchroecker}, \cite{Nawratil2014}): the presence of some
boundary point in the projective closure of the set of self--motions of an
$n$--pod implies precise geometric constraints on base and platform points.
The set of boundary points is a $5$--dimensional complex algebraic variety, but
we give geometric interpretation of its complex points, which form a
$10$--dimensional real algebraic variety. 
These results allow us to provide some necessary conditions for the mobility
of $n$--pods. Eventually we consider the case of $n$--pods with high mobility
(namely with strictly more than one degree of freedom) and we also provide
necessary conditions on the geometry of these devices. 

\medskip
From the kinematic point of view this paper contains the following main
results (see Corollary \ref{corollary:mob1} and Theorem
\ref{theorem:mobility_two}):

\medskip
\noindent
{\bf Result 1} {\it 
If an $n$--pod is mobile, then one of the following conditions holds:
	\begin{itemize}
	\item[(i)]
		There exists at least one pair of orthogonal projections $\pi_L$ and
$\pi_R$ such that the projections of the platform points $p_1, \ldots, p_n$ by
$\pi_L$ and of the base points $P_1, \ldots, P_n$ by $\pi_R$ differ by an
inversion or a similarity.
	\item[(ii)]
		There exists $m \leq n$ such that $p_1, \ldots, p_m$ are collinear and
$P_{m+1}, \ldots, P_n$ are collinear, up to permutation of indices.
	\end{itemize}
}

\medskip
In the following we only give one example for each of the two cases of Result 1,
as a full listing of all examples known in the literature is beyond the scope of
this paper (for some of them, see for example \cite{GeissSchreyer},  
\cite{Karger} and \cite{KargerHusty}).

\smallskip
{\bf \noindent Example ad (i):} 

It was proven by Bricard (cf.\ Chapter VI of \cite{bricard}) that there is exactly one type of non--trivial\footnote{The trivial motions with this property are translations with spherical trajectories, and the rotation of the moving system about a fixed axis or fixed point, respectively.} motions, where all points have spherical paths. 
Moreover it is well known (cf.\ page 324--325 of \cite{bottroth}) that this motion--type is a composition of a rotation about a fixed axis and a translation parallel to this axis. Without loss of generality we can assume that this axis is the $z$--axis. We can take any number of points $p_1, \ldots, p_n$ as platform anchor points, and the centers $P_1, \ldots, P_n$ of the spheres containing their paths as base anchor points. For $n \geq 6$ and generic choice of $p_1, \ldots, p_n$, we will get an $n$--pod with mobility $1$. If we project the points $p_i$ and $P_i$ onto the $xy$--plane, then the resulting points $\pi_z(p_i)$ and $\pi_z(P_i)$ are coupled by  an inversion $\iota$ followed by a rotation $\rho$ and therefore the condition (i) is fulfilled. 

Based on this example we also want to show that the condition (i) is not
sufficient for the existence of a self-motion. 
To do so, we add an extra leg $(p_{n+1}, P_{n+1})$ where $P_{n+1}$ is not the
center of the sphere holding the path of $p_{n+1}$, but another point with the
same $xy$--coordinates. Then the hypothesis of (i) in Result 1
is fulfilled, but the mobility of the new $(n+1)$--pod is zero.

\smallskip
{\bf \noindent Example ad (ii):}

On the contrary, condition (ii) is  sufficient for the existence of a self--motion. 
If the platform is located in a way that the carrier line of $p_1, \ldots, p_m$ coincides with the carrier line of $P_{m+1}, \ldots, P_n$, then the platform can rotate freely about this line. Therefore we get a 1--dimensional set of so--called {\it butterfly--motions}. 

Note that $n$--pods which fulfill the condition (ii) can also have further
self--motions beside these butterfly--motions, even if they do not possess the
property of item (i). Good examples for this fact are the three types of
Bircard's flexible octahedra \cite{bricardoct}, as they can be interpreted as
hexapods.

\medskip
\noindent
{\bf Result 2} {\it 
	Let $\Pi$ be an $n$--pod with mobility $2$ or higher. Then one of the
following holds:
	\begin{itemize}
		\item[(a)]\smallskip
			there are infinitely many pairs $(L,R)$ of elements of $S^2$ such that the
points $\pi_L(p_1), \ldots, \pi_L(p_n)$ and $\pi_R(P_1), \ldots, \pi_R(P_n)$
are equivalent by an inversion or a similarity;
		\item[(b)]
			there exists $m \leq n$ such that $p_1, \ldots, p_m$
are collinear and $P_{m+1} = \ldots = P_n$, up to permutation of indices and
interchange between base and platform;
		\item[(c)]
			there exists $m \leq n$ with $1 < m < n-1$  such that $p_1, \ldots, p_m$
lie on a line $g \subseteq \R^3$ and $p_{m+1}, \ldots, p_n$ lie on a line $g'
\subseteq \R^3$ parallel to $g$, and $P_1, \ldots, P_m$ lie on a line $G
\subseteq \R^3$ and $P_{m+1}, \ldots, P_n$ lie on a line $G' \subseteq \R^3$
parallel to $G$, up to permutation of indices.
	\end{itemize}
}

\smallskip
This last result, in particular condition (a), is the starting point of further
investigations on pentapods with mobility $2$, which are carried on in
\cite{GalletNawratilSchicho}, relying on a new technique called \emph{M\"obius
Photogrammetry}.

\section{Compactification of $\SE$}
\label{compactification}

We start our discussion in Subsection \ref{compactification:new} by introducing
a new compactification of the group of direct isometries of $\R^3$ in a
projective space. Then in Subsection \ref{compactification:action} we study the
natural action of isometries on this compactification, and we prove that it is
given by linear changes of coordinates. At last in Subsection
\ref{compactification:boundary} we describe the boundary of the
compactification.

\subsection{A new compactification}
\label{compactification:new}

We study the 6-dimensional algebraic group $\SE$ of direct isometries of affine
$3$--space $\R^3$ into itself.
One can embed $\SE$ as an open subset of a quadric hypersurface in $\p^7_{\R}$,
called the \emph{Study quadric}. This compactification of $\SE$ turns out to be
extremely useful in the study of mobility properties of objects coming from
robotics and kinematics (see for example \cite{HegeduesSchichoSchroecker},
\cite{Nawratil2014}). However, in our situation we will see that a different
compactification will lead us to a better comprehension of the phenomena which
can arise.

Any isometry of $\R^3$ can be written as a pair $(M,y)$, where $M \in \SO$ is
the linear contribution and $y \in \R^3$ is the image of the origin
$o \in \R^3$. We define $x := -M^t y = -M^{-1} y$ and $r := \langle
x,x \rangle = \langle y,y \rangle$, where $\langle \cdot, \cdot \rangle$ is the
Euclidean scalar product. The isometry $(M, y)$ is considered as a point in
$\p^{16}_{\R}$ with coordinates 
\begin{itemize}
	\item[$\cdot$] $m_{11}, \dots, m_{33}$ (the entries of the matrix),
	\item[$\cdot$] $x_1,\dots,x_3$ (the coordinates of $x$), 
	\item[$\cdot$] $y_1,\dots,y_3$ (the coordinates of $y$), 
	\item[$\cdot$] $r$ and $h$ (a homogenization coordinate; for group elements
we can assume that it is equal to $1$).
\end{itemize}
The group $\SE$ is defined by the inequality $h \ne 0$ and equations
\[ M M^t = M^t M = h^2 \cdot \mathrm{id}_{\R^3}, \;\; \det(M)=h^3, \]
\[ M^t y + h x = 0, \;\; M x + h y = 0 , \]
\[ \langle x,x \rangle = \langle y,y \rangle = r h \]
(not all equations are needed, for instance $M^t y + h x = 0$ is a consequence
of the other equations and the inequality).
We define $X_{\R}$ as the Zariski closure of $\SE$, i.e. the zero set of the set
of all equations vanishing at $\SE$.
Using computer algebra, a Gr\"obner basis for this set of equations can be
computed\footnote{This can be done, for example, by adjoining a temporary
variable $u$, computing a Gr\"obner basis of the equations above and equation
$hu-1$ with an elimination order that eliminates $u$, and then taking the subset
of the basis of elements that have degree $0$ in $u$.}. The degree of $X_{\R}$
can also be computed using computer algebra by the leading monomials of the
Gr\"obner basis: it is $40$.
\begin{remark}
	In the first version of this paper we constructed a projectively isomorphic
compactification of $\SE$ using Study parameters. However it turned out that
the construction above is more direct, computationally simpler and easier to
generalize to higher dimensions.
\end{remark}
We consider the spherical conditions $\left\| \sigma(p_i) - P_i \right\| = d_i$
we want to impose to rigid motions in $\R^3$; if we set $h = 1$, this can
be expressed by:
\begin{equation} 
\label{equation:sphere_condition}
\begin{aligned}
	d_i^2 & = \langle M p_i + y - P_i,M p_i + y - P_i \rangle \\
	& = \langle M p_i,M p_i\rangle + 2\langle M p_i, y \rangle + r +\langle
P_i,P_i\rangle
	-2\langle M p_i, P_i \rangle - 2 \langle y, P_i\rangle \\
	& = \langle p_i, p_i \rangle -\langle P_i, P_i\rangle + r + 2\langle p_i,M^t
y \rangle -2\langle M p_i, P_i \rangle - 2 \langle y, P_i \rangle \\
	& = \langle p_i, p_i\rangle -\langle P_i, P_i \rangle + r - 2\langle
p_i, x\rangle - 2 \langle y, P_i \rangle -2\langle M p_i, P_i\rangle.
\end{aligned}
\end{equation}

\begin{remark}
	After homogenization, Equation~(\ref{equation:sphere_condition}) becomes
linear in the projective coordinates of $\p^{16}_{\R}$. 
\end{remark}

\smallskip
By introducing this compactification of $\SE$ we reduced the problem of dealing
with Equation (\ref{equation:spherical_condition}) to the problem of
understanding linear equations on $X_{\R} \subseteq \p^{16}_{\R}$. In order to
fully use the techniques from algebraic geometry and to be able to set up bond
theory, we need to extend our ground field to the complex numbers.

\begin{definition}
	From now on we work with the complexification of $X_{\R}$, denoted by $X$. In
	order to do this we simply take the equations defining $X$ (which have real
	coefficients), and we think of them as polynomials over $\C$. Hence what we
	get is a projective variety in $\p^{16}_{\C}$ of complex dimension $6$ and
	degree $40$ whose real points are in bijection with the points of $X_{\R}$.
	Inside $X$ we can consider the complexification of $\SE$, which we will denote
	by $\SEC$, hence we have an injective map $\Phi: \SE \hookrightarrow \SEC
	\subseteq X$.
\end{definition}

\subsection{Action of $\SE$ on $X$}
\label{compactification:action}

In this subsection we want to extend the natural actions of $\SE$ on itself,
given by composition on the left and on the right, to actions of $\SE$ on
$\p^{16}_{\C}$ which restrict to actions on $X$. This will be useful
in the proofs of Section \ref{bond_geometric}, because it will allow us to
exploit the symmetries of $X$. 

Let $\sigma_1: v \mapsto M_1 v + y_1$ and $\sigma_2: v \mapsto M_2 v + y_2$ be
isometries. Then the product $\sigma_{12} = \sigma_1 \sigma_2$ maps 
\[ v \; \mapsto \; (M_1 M_2) v + (M_1 y_2 + y_1). \]
We set $M_{12} := M_1 M_2$ and $y_{12} := M_1 y_2 + y_1$. The remaining affine
coordinates of $\sigma_{12}$ are
\[ 
	\begin{spreadlines}{0.8em}
		\begin{aligned}
			x_{12} & = -M_{12}^t y_{12} = -M_2^t M_1^t M_1 y_2 - M_2^t M_1^t y_1 \\
			& = -M_2^{t} y_2 - M_2^t M_1^t y_1 = x_2 + M_2^t x_1, \\ 
			r_{12} & = \langle y_{12}, y_{12}\rangle = \langle y_1,y_1\rangle +
		\langle{M_1 y_2, M_1 y_2}\rangle + 2 \langle M_1 y_2, y_1\rangle \\
			& = r_1 + r_2 - 2 \langle x_1,y_2 \rangle.
		\end{aligned} 
	\end{spreadlines} 
\]
This product becomes bilinear after homogenization, and the projective coordinates of $\sigma_{12}$ are
\begin{equation} 
\label{equation:product}
	\big( \underbrace{h_1 h_2}_{h_{12}}: \underbrace{M_1
	M_2}_{M_{12}}:
	\underbrace{M_2^t x_1 + h_1 x_2}_{x_{12}}: \underbrace{h_2 y_1
	+ M_1 y_2}_{y_{12}}: \underbrace{h_2 r_1 + h_1 r_2 - 2\langle
	x_1,y_2\rangle}_{r_{12}} \big).
\end{equation}
(The matrices and vectors appearing in the above coordinates should be replaced
by their entries.) 

\begin{proposition}
\label{proposition:product}
	Let $(h_1:M_1:x_1:y_1:r_1),(h_2:M_2:x_2:y_2:r_2) \in X$. Then the above
product	is defined if at least one of $h_1$, $h_2$ is not equal to zero. 
\end{proposition}

\begin{proof}
	Assume that the product is undefined, which means that all entries are zero.
	In particular, $h_1 h_2 = 0$. We assume $h_1 = 0$ (the other case $h_2 = 0$
	can be treated analogously). Assume, indirectly, that $h_2\ne 0$. Since
	$\det(M_2) = h_2^3$, it follows that $M_2$ is invertible.
	Since $M_1 M_2 = 0$, it follows that $M_1=0$. Since $M_2^t x_1 + h_1 x_2=0$,
	it follows that $x_1 = 0$. Since $h_2 r_1 + h_1 r_2 - 2 \langle x_1,
	y_2 \rangle = 0$, it follows that $r_1 = 0$. Then all coordinates of the first
	element vanish,	a contradiction.
\end{proof}

\begin{remark}
	Since the $h$--coordinate of any element $\sigma \in \Phi(\SE)$ is always
different from zero, then Proposition \ref{proposition:product} ensures that
left and right multiplication by $\sigma$, which a priori are maps from
$\Phi(\SE)$ to itself, extend to linear maps $X \longrightarrow X$.
\end{remark}

\begin{remark}
\label{remark:action}
	We specialize Equation \eqref{equation:product} to the cases of left and
right multiplication by translations or rotations along the origin. We fix
$\sigma \in X$ and we suppose that it has coordinates $\sigma = (h: M: x: y:
r)$. 
\begin{adjustwidth}{0.4cm}{0.4cm}
\begin{itemize}
	\item[a)]
		Given a vector $s \in \R^3$, the translation by $s$ gives the following
		element	$\sigma' \in \Phi(\SE)$:
		\[ \sigma' = \big( 1: \mathrm{id}: -s: s: \langle s, s \rangle \big). \]
		Then left multiplication by $\sigma'$ provides
		\[ 
			\sigma' \tth \sigma \; = \; \big( h: M: -M^t s + x: hs + y: h \tth
			\langle s, s \rangle + r + 2 \langle s, y \rangle \big),
		\]
		while right multiplication by $\sigma'$ provides
		\[
			\sigma \tth \sigma' \; = \; \big( h: M: x - hs: y + Ms: r + h \tth \langle
			s, s \rangle - 2 \langle x, s \rangle \big).
		\]
	\item[b)]
		Given an orthogonal matrix $M' \in \SO$, the rotation around the origin by
		$M'$ gives the following element $\sigma' \in \Phi(\SE)$:
		\[ \sigma' \; = \; (1: M: 0: 0: 0). \]
		Then left multiplication by $\sigma'$ provides
		\[
			\sigma' \tth \sigma \; = \; \big( h: M'M: x: M'y: r \big)
		\]
		while right multiplication by $\sigma'$ provides
		\[
			\sigma \tth \sigma' \; = \; \big( h: M M': M'x: y: r \big)
		\]
\end{itemize}
\end{adjustwidth}
\end{remark}

\subsection{Boundary of $X$.}
\label{compactification:boundary}

\begin{definition}
The \emph{boundary} of $X$ is defined as $B := X \setminus \SEC$. It is the
closed subset of $X$ cut out by the linear equation $h = 0$. 
\end{definition}

For any point $(h: M: x: y: r) \in B$, we have 
\[ M M^t \; = \; M^t M \; = \; 0, \quad \quad M^t y \; = \; M x \; = \; 0, \]
\[ \langle x,x \rangle \; = \; \langle y,y \rangle \; = \; 0 . \]
The first equation shows that $\mathrm{rank}(M) \le 1$. Hence $M = v w^t$ for
two suitable vectors $v, w \in \C^3$. It should be noted that $v$ and $w$ are
not unique: one may multiply $v$ by a nonzero complex number and $w$ by its
inverse. We have two cases:
\begin{itemize}
	\item[i.]
		if $M = 0$, we can take both $v$ and $w$ to be zero;
	\item[ii.]
		if $M \neq 0$, then again by the same equation, it follows that
$\langle v,v \rangle = \langle w,w \rangle = 0$. The second equation implies
$\langle x,w \rangle = \langle y,v \rangle = 0$. Then the subspace spanned by
$x$ and $w$ is totally isotropic with respect to $\langle \cdot, \cdot \rangle$,
and this implies it has dimension $1$, so $x$ and $w$ are linearly dependent.
Similarly, also $y$ and $v$ are linearly dependent.
\end{itemize}

\medskip
We can partition the boundary in five subsets. The nomenclature of the various
types of points will become clear in Section \ref{bond_geometric}.

\subsubsection{Vertex.}
\label{compactification:boundary:vertex}

For any real point in $B$, we have $v = w = x = y = 0$. The only nonzero
coordinate is $r$, so we have a unique real point $v_0 = (0: \ldots : 0 : 1)$ in
$B$, called the \emph{vertex}. A computer algebra computation shows that this is
a point of multiplicity $20$ on $X$, but we do not need this fact.

\subsubsection{Inversion Points.}
\label{compactification:boundary:inversion}

Consider the matrix $N := r M + 2y x^t$ (it will be clear later in the
discussion why we choose this expression for $N$). A boundary point $\beta$ with
$M \ne 0$ and $N \ne 0$ is called an \emph{inversion point}. In this case we
have $x = \lambda w$ and $y = \mu v$ with $\lambda, \mu \in \C$. Hence the
coordinates of an inversion point can be written as $(0:v w^t:\lambda w: \mu
v:r)$. Since $v$
and $w$ satisfy the quadratic equation $\langle v,v \rangle = \langle w,w
\rangle = 0$ (called the equation of the \emph{absolute conic} in $\p^2_{\C}$),
the complex dimension of the set of inversion points is 5 (one for $v$, one for
$w$, one for $\lambda$, one for $\mu$, one for $r$). One can show that these are
smooth points of the boundary, but we do not need this fact. 

In order to compute normal forms, we first apply rotations. Multiplication from
the right  by a rotation of matrix $M'$ gives (see Remark
\ref{remark:action}) 
\[ (0: v w^t M': \lambda \tth M' w: \mu \tth v: r), \]
so it leaves $v$ fixed. Being $M'$ an orthogonal
matrix, it is in particular unitary, so it preserves both the scalar product
and the Hermitian norm of $w$, and the action is transitive on vectors with
$\langle w, w \rangle = 0$ and of the same Hermitian norm. Hence $w$ can be
taken to a vector of the form $\delta \tth (1, i, 0)^t$, where $\delta \in
\C^{*}$. Multiplication from the left acts analogously on $v$. Hence by suitable
rotations from both sides we obtain $v = \gamma \tth (1,i,0)^t$ and $w =
\delta \tth (1,i,0)^t$ with both $\gamma$ and $\delta$ different from zero
since $M \neq 0$. Then projectively we can suppose that $M = \left(
\begin{smallmatrix}
	1 & i & 0 \\
	i & -1 & 0 \\
	0 & 0 & 0 
\end{smallmatrix} 
\right)$. 

The action of left multiplication by a translation with vector $s \in \R^3$ on
the boundary point $\beta$ gives (see Remark \ref{remark:action})
\[ \big(0:vw^t:(-\langle v,s \rangle + \lambda) \tth w: \mu \tth v: r + 2\mu
\tth \langle v,s\rangle
\big) , \]
and similarly the action by right multiplication with vector $t \in \R^3$ gives
\[ \big(0:vw^t: \lambda \tth w: (\langle w,t \rangle + \mu) \tth v: r - 2\lambda
\tth \langle w,t \rangle \big) . \]
This shows that we can achieve $\lambda = \mu = 0$ by multiplication by
translations from both sides (for example, since we reduced to the situation $v
= (1, i,0)^t$, one can take $s_1 = \mathrm{Re} \lambda$, $s_2 = \mathrm{Im}
\lambda$ and $s_3$ to be arbitrary, where $s = (s_1, s_2, s_3)^{t}$, and
similarly for $t$). It also shows that the matrix $N$ is invariant under
translations (this was the reason why we chose $N$ in this way). So by
translations from both sides, we obtain $x = y = 0$. The value of $r$ cannot be
changed by any rotation that fixes $x = y = 0$, but we still can apply a
rotation of the form $\left( 
\begin{smallmatrix}
	c & d & 0 \\
	-d & c & 0 \\
	0 & 0 & 1 
\end{smallmatrix} 
\right)$, with $c^2 + d^2 = 1$, from the left. The effect on $M$ is
multiplication by $(c+id)$, and we have no effect on $r$. Projectively, this
is the same as leaving $M$ untouched and multiplying $r$ by $(c+id)^{-1}$.
Hence we can reach the situation where $r \in \R_{>0}$. We notice that $r$
cannot be zero, otherwise we would have $N = 0$. So inversion points have the
following normal forms:
\[ \beta \; = \; (0: \underbrace{1:i:0:i:-1:0:0:0:0}_M:
\underbrace{0:0:0}_{x}: \underbrace{0:0:0}_{y}: r), \]
with $r\in\R_{>0}$.

\subsubsection{Butterfly Points.}
\label{compactification:boundary:butterfly}

A boundary point $\beta$ with $M \ne 0$ and $N = 0$ is called a \emph{butterfly
point}. The complex dimension of the set of butterfly points is $4$: as before,
we can choose $v$ and $w$ on the absolute conic curve, and $\lambda, \mu \in
\C^\ast$.
The normal form is constructed similarly as above. In this case, when we obtain
$x = y = 0$, the fact that $M \neq 0$ and $N = 0$ forces $r$ to be zero. In this
case, we have only a single normal form, namely 
\[ \beta = (0: \underbrace{1:i:0:i:-1:0:0:0:0}_M: \underbrace{0:0:0}_{x}:
\underbrace{0:0:0}_{y}: 0). \]

\subsubsection{Similarity Points.}
\label{compactification:boundary:similarity}

The points $\beta = (0: M: x: y: r) \in B$ such that $M = 0$, $x \ne 0$ and $y
\ne 0$ are called \emph{similarity points}. Since $x$ and $y$ are on the
absolute conic, the complex dimension of the set of similarity points is $4$.

\smallskip
To compute normal forms of similarity points, we first apply rotations. As
we saw in Subsection \ref{compactification:boundary:inversion}, right
multiplication fixes $y$ and $r$ and can transform $x$ to $\gamma
\tth (1,i,0)^t$, and left multiplication fixes $x$ and $r$ and can transform
$y$ to $\delta \tth (1, i, 0)^t$, with both $\gamma$ and $\delta$ in $\C^{*}$.
Hence projectively we can always suppose that $\delta = 1$, so we can
reduce any similarity point to one such that $x = \gamma \tth (1,i,0)^t$ and $y
=
(1,i,0)^t$. Then translations act transitively on $r$, thus we may
get to the situation with $r = 0$. Eventually, as we have already seen in
Subsection \ref{compactification:boundary:inversion}, we can perform rotations
so that we can ensure that $\gamma$ is a real positive number. So we get normal
forms of the following kind
\[ \beta = (0: \underbrace{0:0:0:0:0:0:0:0:0}_M:
\underbrace{\gamma:i\gamma:0}_{x}: \underbrace{1:i:0}_{y}: 0), \]
with $\gamma \in \R_{>0}$.

\subsubsection{Collinearity Points.}
\label{compactification:boundary:collinearity}

For the last group of points $\beta$ in $B$ we have $M = 0$ and either $x = 0$,
$y \ne 0$ or $x \ne 0$, $y = 0$. These points are called \emph{collinearity
points}. There are two subsets of collinearity points, one with $x = 0$ and one
with $y = 0$. Both subsets have complex dimension $2$ (since there is still a
free value for $r$ to choose). 

By rotations, we can achieve either $x = (1:i:0)^t$ or $y = (1:i:0)^t$.
Translations act transitively on $r$, so we get two normal forms, namely
\[ \beta = (0: \underbrace{0:0:0:0:0:0:0:0:0}_M: \underbrace{1:i:0}_{x}:
\underbrace{0:0:0}_{y}:0) , \]
\[ \beta = (0: \underbrace{0:0:0:0:0:0:0:0:0}_M: \underbrace{0:0:0}_{x}:
\underbrace{1:i:0}_{y}: 0). \]

\bigskip
We conclude the section about the boundary of $X$ by showing that we can
associate to each inversion, butterfly and similarity point a pair $(L,R)$ of
elements of $S^2$, namely oriented directions in $\R^3$. This piece of
information will play an important role in the main results of Section
\ref{bond_geometric}. 

\smallskip
We recall that in the case of inversion and butterfly points the matrix $M$ is
of rank $1$, since it is non zero and has rank $\leq 1$, as implied by the
boundary condition $h = 0$. Hence $M$ is of the form $v w^t$ for two non zero
vectors whose coordinates satisfy 
\[ v_1^2 + v_2^2 + v_3^2 \; = \; 0 \quad \mathrm{and} \quad w_1^2 + w_2^2 +
w_3^2 \; = \; 0 \]
Then we can think of $v$ and $w$ as points of the conic $C = \big\{ \alpha^2 +
\beta^2 + \gamma^2 = 0 \big\}$ in $\p^2_{\C}$. Notice that, although $v$ and $w$
are not unique, they always give the same pair of points on the conic. We would
like to think of $v$ and $w$ as directions in $\R^3$, and in order to do this we
provide an identification between $C$ and $S^2$. This identification is
accomplished in two steps, namely first we consider an isomorphism $C \cong
\p^1_{\C}$ and then we take the bijection between $\p^1_{\C}$ and $S^2$ given by
the stereographic projection\footnote{These identifications are very special
ones, since they become isomorphisms of real varieties when we consider,
respectively, componentwise complex conjugation on $C$, the map $(s,t) \mapsto
(-\overline{t}, \overline{s})$ on $\p^1_{\C}$ and the antipodal map on $S^2$
as real structures. This can be understood as a hint why these particular
choices work well, but we do not use this property in our investigations.}. \\
The isomorphism $C \cong \p^1_{\C}$ is given by the parametrization
\[ \p^1_{\C} \ni (s,t) \; \mapsto \; \big( (s^2 - t^2): i(s^2 + t^2): 2st
\big) \in C \]
and its inverse
\[
	\left\{ 
	\begin{array}{lcll}
		(\alpha: \beta : \gamma) & \mapsto & (\alpha - i \beta : \gamma) &
\mathrm{if\ } (i\alpha +
\beta, \gamma) \neq (0,0) \\
		(\alpha: \beta : \gamma) & \mapsto & (\gamma: - \alpha - i \beta) &
\mathrm{otherwise}
	\end{array}
	\right.
\]
The identification between $\p^1_{\C}$ and $S^2$ by stereographic projection
is provided by the following equations:
\[
\begin{array}{l}
	\left\{
	\begin{array}{lcll}
		(0,0,1) & \mapsto & (0:1) \in \p^1_{\C} \\ [2mm]
		(\lambda, \mu, \nu) & \mapsto & \left( 1: \frac{\lambda + i \mu}{1 - \nu}
\right) \in \p^1_{\C} & \mathrm{for\ all\ } (\lambda, \mu, \nu) \in S^2
\setminus \big\{ (0,0,1) \big\}
	\end{array}
	\right. \\ [8mm]
	\left\{
	\begin{array}{lcll}
		(0:1) & \mapsto & (0,0,1) \in S^2 \\ [2mm]
		(1: a + ib) & \mapsto & \left( \frac{2a}{a^2 + b^2 + 1}, \frac{2b}{a^2 +
b^2 + 1}, \frac{a^2 + b^2 - 1}{a^2 + b^2 + 1}  \right) & \mathrm{for\ all\ } a,b
\in \R
	\end{array}
	\right.
\end{array}
\]
For example if $v = (1:i:0) \in C$, then the corresponding element
of $S^2$ is the South pole $(0,0,-1)$.

\smallskip
In this way it is possible to assign to each inversion or butterfly point a
pair $(L,R)$ of elements in $S^2$. We would like to do the same for similarity
points. There is a glaring obstruction in doing this, namely the fact that for
similarity points both the $h$ and the $m_{ij}$--coordinates are zero. On the
other hand for all boundary points the two matrices $M$ and $xy^t$ are linear
dependent, and in the case of similarity points $x y^t$ is different from zero.
Moreover $x$ and $y$ satisfy $\langle x, x \rangle = \langle y, y \rangle = 0$.
So we can associate to a similarity point the pair of elements of $S^2$ coming
from the vectors $x$ and $y$. 

\begin{definition}
\label{definition:left_right}
	Via these identifications we can associate to every inversion, butterfly or
similarity point $\beta$ in $B$ a pair $(L,R)$ of elements of $S^2$,
which are respectively called the \emph{left} and the \emph{right vector} of
$\beta$.
\end{definition}

\section{Geometric Interpretation of Bonds}
\label{bond_geometric}

This section represents an instance of a more general technique
called \emph{bond theory}: the goal is to extract information from boundary
points which arise as limits of self--motions of an $n$--pod. Boundary points do
not represent direct isometries of $\R^3$, but nevertheless we can give them
geometric meaning, since their presence as limits of self--motion determine
geometric conditions the base and platform points have to satisfy.

\medskip
Recall from Section \ref{compactification} that $\Phi(\SE)$ is an embedding of
$\SE$ in $\p^{16}_{\C}$, that we denoted by $\SEC$ its complexification and
that we defined $X$ as the Zariski closure of $\SEC$. Moreover, recall that we
think of an $n$--pod as a triple
		\[ \Pi = \Big( (p_1, \dots, p_n),(P_1, \dots, P_n),(d_1, \dots, d_n) \Big)
\]
Eventually, recall that in the new coordinates of $\p^{16}_{\C}$ the spherical
condition given by Equation \eqref{equation:spherical_condition} reads as
\begin{equation}
\label{equation:leg_condition}
	d_i^2 \tth h \; = \; \big( \langle p_i, p_i\rangle \tth - \langle P_i, P_i
\rangle \big) \tth h + r - 2 \langle p_i, x\rangle - 2 \langle y, P_i \rangle
-2 \langle M p_i, P_i \rangle,
\end{equation}
which gives a linear form $l_i$ on $\p^{16}_{\C}$.

\begin{remark}
	Let $\Pi$ be an $n$--pod, then the real points of the intersection
	\[ \SEC \cap \big\{ l_i \, = \, 0 \big\} \]
	are in bijective correspondence with the set of all $\sigma \in \SE$ such that
the distance between $\sigma(p_i)$ and $P_i$ is $d_i$, namely the set of direct
isometries satisfying the spherical condition for $p_i$ and $P_i$.
\end{remark}

\begin{definition}
	Let $\Pi$ be an $n$--pod, then the intersection of $X$ with the hyperplanes
defined by $\big\{ l_i = 0 \big\}$ for $i \in \{ 1, \ldots, n \}$ is called
the \emph{complex configuration set} of $\Pi$ and denoted by $K_\Pi$; the real
points of this intersection are called the \emph{real configuration set} of
$\Pi$.
The complex dimension of $K_\Pi \cap \SEC$ as a complex algebraic variety is
called the \emph{mobility} of $\Pi$. If the mobility of $\Pi$ is greater than or
equal to $1$, then $\Pi$ is said to be \emph{mobile}.
\end{definition}

\begin{remark}\label{rem:cardinality}
	For a generic hexapod, the complex configuration set is finite of
cardinality~$40$. This has been shown by \cite{RongaVust}. It also follows from
the fact $\deg(X) = 40$ that was mentioned in Subsection
\ref{compactification:new}: in fact in the coordinates of $\p^{16}_{\C}$
the spherical equation \eqref{equation:sphere_condition} becomes linear, hence
every leg of an $n$--pod imposes a linear condition on $X$. For a
generic hexapod $\Pi$, its complex configuration set $K_{\Pi}$ is given by
the intersection of $6$ generic hyperplanes in $\p^{16}_{\C}$ with $X$. The
intersection of the hyperplanes gives a generic codimension $6$ linear space
$H_{\Pi}$. Now we use the following general fact from projective geometry: the
intersection of a complex projective variety of dimension $r$ and degree $d$
with a generic linear space of codimension $r$ consists of $d$ points. Since $X$
has dimension $6$, its intersection with $H_{\Pi}$, namely $K_{\Pi}$, is given
by a finite number of points whose cardinality equals the degree of $X$. Hence
in the generic case $K_{\Pi}$ is constituted by $40$ points. 
\end{remark}

\begin{definition}
	Let $\Pi$ be an $n$--pod, we define its set of \emph{bonds} $B_{\Pi}$ as the
intersection of $K_{\Pi}$ and the hyperplane $\big\{ h = 0 \big\}$, namely
$B_{\Pi}$ is the intersection of $K_{\Pi}$ with the boundary $B$ of $X$, as
defined in Subsection \ref{compactification:boundary}.
\end{definition}

\begin{remark}
	For bonds, Equation \eqref{equation:leg_condition} reduces to
	\begin{equation}
	\label{equation:leg_condition_bonds}
		r - 2 \langle p_i, x\rangle - 2 \langle y, P_i \rangle
-2 \langle M p_i, P_i \rangle \; = \; 0.
	\end{equation}
\end{remark}

\begin{definition}
	We call the condition imposed by Equation
\eqref{equation:leg_condition_bonds} the \emph{pseudo spherical condition} for
the points $(p_i, P_i)$ at the bond $(0:M:x:y:r)$.
\end{definition}

\begin{remark}
\label{remark:vertex_no_bond}
	Recall that the vertex $v_0$ is the only real point of $B$ (see Subsection
\ref{compactification:boundary:vertex}). Since $v_0$ can never be a bond of an
$n$--pod (in fact, by instantiating $v_0$ in Equation
\eqref{equation:leg_condition_bonds} we would get the contradiction $1 = 0$),
then $B_{\Pi}$ has no real points.
\end{remark}

\begin{remark}
\label{remark:mobility_bonds}
	If an $n$--pod $\Pi$ is mobile, then by definition $\dim K_{\Pi} \cap \SEC
\geq 1$, so $\dim K_{\Pi} \geq 1$. Since $B_{\Pi}$ is an hyperplane section of
$K_{\Pi}$, it follows that the dimension decreases at most by $1$, so $B_{\Pi}$
is not empty. By the same argument we have that if the mobility is greater
than, or equal to $2$, then $\Pi$ admits infinitely many bonds.
\end{remark}

Before coming to the main results of this section, recall that at the end of
Subsection \ref{compactification:boundary} we associated to each inversion,
butterfly and similarity points a pair of directions in $S^2$, called the left
and right vector of the boundary point.

\begin{definition}
	Given a unit vector $\epsilon \in S^2$, we denote by $\pi_{\epsilon}: \R^3
\longrightarrow \R^2$ the \emph{orthogonal projection along $\epsilon$}, namely
for every $p = (a,b,c) \in \R^3$, the point $\pi_{\varepsilon}(p)$ is the
orthogonal projection of $p$ on the affine plane orthogonal to
$\varepsilon$, passing through the origin.
\end{definition}

\smallskip
We are ready to state and prove the main results of this section.
\begin{theorem}
\label{theorem:inversione-similarity}
	There is a one-to-one correspondence between inversion/simi-larity points
	$\beta$ with both left and right vectors $L$ and $R$ equal to the South pole
	$(0,0,-1) \in S^2$ and inversions/similarities $\kappa$ of the plane such that
	for any pair of points $(p,P)$ in $\R^3$ the pseudo spherical condition for
	$(p,P)$ at $\beta$ is equivalent to the fact that $\kappa(q) = Q$ where $q 	=
	\pi_L(p)$ and	$Q = \pi_R(P)$. 
\end{theorem}
\begin{proof}
	We first treat the case of inversion points. Suppose that $\beta_{0} \in B$ is
an inversion point with $L = R = (0,0,-1)$. Suppose furthermore that $\beta_{0}$
is in the normal form (see Subsection
\ref{compactification:boundary:inversion}):
	\[ \beta_{0} \; = \; (0:\underbrace{1:i:0:i:-1:0:0:0:0}_M:
\underbrace{0:0:0}_{x}: \underbrace{0:0:0}_{y}: r), \]
	with $r \in \R_{>0}$. We get that $\pi_{L} = \pi_{R}$ are the projection on
the first two coordinates. Thus if $p = (a, b, c)$ and $P = (A, B, C)$, then
$q = (a, b)$ and $Q = (A, B)$. If we instantiate the pseudo spherical
condition for $(p, P)$ at $\beta_0$ given by Equation
(\ref{equation:leg_condition_bonds}) we get the relations:
	\begin{equation}
	\label{equation:inversion}
		\left\{
		\begin{array}{l} 
			a A - b B = \nicefrac{r}{2} \\
			b A + a B = 0
	  \end{array}
		\right.
	\end{equation}
	which define an inversion $\kappa_0$ mapping $q$ to $Q$. Conversely, suppose
we are given an inversion $\kappa_0$ described by Equation
\eqref{equation:inversion}.
Then going backwards in the previous argument we can see that we obtain an
inversion point in normal form as in the thesis. \\
	Suppose now that the $\beta \in B$ is an inversion point with $L = R =
(0,0,-1)$, but not necessarily in normal form. Then, as we saw in
Subsection \ref{compactification:boundary:inversion}, we can find two isometries
$\sigma_1, \sigma_2 \in \SE$ which fix left and right vectors such that
$\sigma_1 \beta \sigma_2 = \beta_0$ is in normal form (here $\sigma_1 \beta
\sigma_2$ denotes the element of $X$ obtained by left action by $\sigma_1$ on
$\beta$ and then by right action of $\sigma_2$). Moreover $\sigma_1$ and
$\sigma_2$ induce isometries $\tau_1$ and $\tau_2$ of $\R^2$ such that the
following two diagrams commute:
\[ \xymatrix{\R^3 \ar[r]^{\sigma_1} \ar[d]_{\pi_L} & \R^3 \ar[d]^{\pi_L} \\ \R^2
\ar[r]^{\tau_1} & \R^2 } \qquad \qquad \xymatrix{\R^3 \ar[r]^{\sigma_2}
\ar[d]_{\pi_R} & \R^3 \ar[d]^{\pi_R} \\ \R^2 \ar[r]^{\tau_2} & \R^2 } \]

If $\kappa_0$ is the inversion associated to $\beta_0$, then we define $\kappa =
\tau_1 \kappa_0 \tau_2$, and one can check that the association $\beta
\leftrightarrow \kappa$ satisfies the requirements of the thesis.

	\smallskip
	We consider now the case of similarity points. Suppose that $\beta_0 \in B$
is a similarity point with $L = R = (0,0,-1)$. Suppose furthermore that
$\beta_{0}$ is in the normal form (see Subsection
\ref{compactification:boundary:similarity}):
	\[ \beta = (0: \underbrace{0:0:0:0:0:0:0:0:0}_M:
	\underbrace{\gamma:i\gamma:0}_{x}: \underbrace{1:i:0}_{y}: 0), \]
	with $\gamma \in \R_{>0}$. Again for this kind of points $\pi_L$ and $\pi_R$
are both the projection on the first two coordinates. Performing analogous
computations as before we get the relations:
	\begin{equation}
	\label{equation:similarity}
		\left\{
		\begin{array}{l}
			A = -\gamma \tth a  \\
			B = -\gamma \tth b
	  \end{array}
		\right.
	\end{equation}
	These define a similarity $\kappa_0$ mapping $q$ to $Q$. Conversely, and in
the general case of points not in normal form, we argue as for inversion points.
\end{proof}
\begin{remark}
	As pointed out in Subsection \ref{compactification:boundary:inversion}, the
complex dimension of the set of inversion points is $5$. Theorem
\ref{theorem:inversione-similarity} allows, as remarked in the Introduction, to
associate to it a real dimension, which can be computed as follows: $2$ degrees
of freedom for choosing the vector $L$ and the same for $R$, and $6$ degrees of
freedom for specifying an inversion from $\R^2$ to itself. So, in total, we get
$10$. We can argue analogously for similarity points: here the real dimension
is $8$.
\end{remark}
\begin{cor} 
\label{cor:inversion-similarity}
	Assume that $\beta \in B_{\Pi}$ is an inversion/similarity bond of $\Pi$.
	Let $L,R \in S^2$ be the left and right vector of $\beta$.
	For $i = 1, \ldots, n$, set $q_i = \pi_L(p_i)$ and $Q_i = \pi_R(P_i)$. 
	Then there is an inversion/similarity of $\R^2$ mapping $q_1, \ldots, q_n$ to $Q_1, \ldots, Q_n$.

	Conversely, let $L,R \in S^2$ be two unit vectors such that the images of
$(p_1, \dots, p_n)$ under $\pi_L$ and of $(P_1, \dots, P_n)$ under $\pi_R$
differ by an inversion/ similarity. 
	Then $\Pi$ has an inversion/similarity bond with left vector $L$ and right
vector $R$.
	\end{cor}
\begin{proof}
	In both cases of inversion and similarity points we can apply suitable
isometries in order to put $\beta$ in normal form. Then it is enough to apply
Theorem \ref{theorem:inversione-similarity}.
\end{proof}

\begin{theorem}
\label{theorem:butterfly}
	There is a one-to-one correspondence between butterfly points $\beta$ and
	pairs $(g_L, g_R)$ of oriented lines in $\R^3$ such that for any pair of
	points $(p,P)$ in $\R^3$ the pseudo spherical condition for $(p,P)$ at $\beta$
	is equivalent to the fact that $p \in g_L$ or $P \in g_R$.
\end{theorem}
\begin{proof}
	Suppose that $\beta_0 \in B$ is a butterfly point in the normal form (see
Subsection \ref{compactification:boundary:butterfly}): 
	\[ \beta_0 = (0: \underbrace{1:i:0:i:-1:0:0:0:0}_M: \underbrace{0:0:0}_{x}:
\underbrace{0:0:0}_{y}: 0). \]
	In this case we associate to $\beta_0$ the lines $g_L = g_R = \{
z\mathrm{-axis} \}$, both oriented to the South pole $(0,0,-1) \in S^2$. If we
instantiate the pseudo spherical condition for $(p, P)$ at $\beta$ given by
Equation (\ref{equation:leg_condition_bonds}) we get the relations:
	\begin{equation}
	\label{equation:butterfly}
		\left\{
		\begin{array}{l}
			a A - b B = 0 \\
			a B + b A = 0
	  \end{array}
		\right.
	\end{equation}
	Equation (\ref{equation:butterfly}) can be interpreted as: the vector $(a,
b)$ is parallel both to the vector $(A, -B)$ and to the vector $(B,
A)$. This is possible if and only if either $(a, b) = (0,0)$ or $(A,
B) = (0,0)$. Hence either $p$ is of the form $(0,0,c)$ (namely it lies on
$g_L$) or $P$ is of the form $(0,0,C)$ (namely it lies on $g_R$). 

	If $\beta \in B$ is an arbitrary butterfly point, then from Subsection
\ref{compactification:boundary:butterfly} we know that there exist isometries
$\sigma_1, \sigma_2 \in \SE$ such that $\sigma_1 \beta \sigma_2 = \beta_0$ is
in normal form. Then we associate to $\beta$ the pair of lines 
\[ (g_L, g_R) \; = \; \Big( \left(\sigma_1\right)^{-1}\big( \{ z\mathrm{-axis}
\} \big), \left(\sigma_2\right)^{-1}\big( \{ z\mathrm{-axis} \} \big) \Big) \]
	with orientation given by the left and right vectors of $\beta$. One can
check that the equivalence in the thesis holds. Conversely, given two oriented
lines $g_L$ and $g_R$ we can find isometries $\sigma_1, \sigma_2 \in \SE$ such
that $g_L = \sigma_1 \big( \{ z\mathrm{-axis} \} \big)$ and $g_R = \sigma_2
\big( \{ z\mathrm{-axis} \} \big)$, both oriented to the South pole $(0,0,-1)
\in S^2$. Then we associate to $(g_L, g_R)$ the butterfly point $\sigma_1
\beta \sigma_2$.
\end{proof}

\begin{cor} 
\label{cor:butterfly}
	Assume that $\beta \in B_{\Pi}$ is a butterfly bond of $\Pi$.
	Let $L,R \in S^2$ be the left and right vector of $\beta$.
	Then, up to permutation of indices $1, \dots, n$, there exists $m \leq n$ such
	that $p_1, \ldots, p_m$ are collinear on a line parallel to $L$,
	and $P_{m+1}, \ldots, P_n$ are collinear on a line parallel to $R$.

	Conversely, let $L,R \in S^2$ be two unit vectors such that 
	$p_1, \ldots, p_m$ are collinear on a line parallel to $L$, and
	$P_{m+1}, \ldots, P_n$ are collinear on a line parallel to $R$.
	Then $\Pi$ has a butterfly bond with left vector $L$ and right vector~$R$.
\end{cor}

\begin{notation}
	Recall from Subsection \ref{compactification:boundary:collinearity} that the
set of collinearity points is partitioned into two subsets: if the
$y$--coordinate of a collinearity point is zero we call it a \emph{left
collinearity point}, while if the $x$--coordinate is zero we call it a
\emph{right collinearity point}.
\end{notation}

\begin{theorem}
\label{theorem:collinear}
	There is a one-to-one correspondence between left (resp. right) collinearity
	points $\beta$	and oriented lines $g$ in $\R^3$ such that for any pair of
	points $(p, P)$ in	$\R^3$ the pseudo spherical condition for $(p, P)$ at
	$\beta$ is equivalent to	the fact that $p \in g$ (resp. $P \in g$).
\end{theorem}
\begin{proof}
	Suppose that $\beta_0 \in B$ is a left collinearity point and suppose that it
is in normal form (see Subsection
\ref{compactification:boundary:collinearity}):
	\[ \beta = (0: \underbrace{0:0:0:0:0:0:0:0:0}_M: \underbrace{1:i:0}_{x}:
	\underbrace{0:0:0}_{y}:0) , \]
	We associate to $\beta_0$ the line $g = \{ z\mathrm{-axis} \}$, directed to
the South pole $(0,0,-1) \in S^2$. If we instantiate the pseudo spherical
condition for $(p, P)$ at $\beta_0$ given by Equation
(\ref{equation:leg_condition_bonds}) we get the relations:
	\[ 0 \; = \; -2(a + i b) \quad \Leftrightarrow \quad a = b = 0 \]
	which is equivalent to $p \in g$. 

	If $\beta \in B$ is an arbitrary left collinearity point we proceed as in the
proof of Theorem \ref{theorem:butterfly}. Analogous arguments prove the
statement about right collinearity points.
\end{proof}

\begin{cor} 
\label{cor:collinear}
	Assume that $\beta \in B_{\Pi}$ is a collinearity bond of $\Pi$.
	Then either $p_1, \ldots, p_n$ are collinear or $P_1, \ldots, P_n$ are collinear (or both).

	Conversely, if $p_1,\dots,p_n$ are collinear or $P_1, \ldots, P_n$ are
collinear (or both), then $\Pi$ has a collinearity bond.
\end{cor}

The following Corollary gives a necessary criterion for mobility of $n$--pods. 
For the fist time (to the authors' knowledge) a necessary criterion
for the mobility of hexapods can be defined by the invariant
linkage parameters, irrespective of a specific configuration. (The well-known
criterion for infinitesimal mobility, see \cite{Merlet89}, refers to an explicit
relative pose of the platform with respect to the base.)

\begin{cor}
\label{corollary:mob1}
If an $n$--pod is mobile, then one of the following conditions holds:
	\begin{enumerate}[(i)]
		\item
			There exists at least one pair of orthogonal projections $\pi_L$ and
$\pi_R$ such that the projections of the platform points $p_1, \ldots, p_n$ by
$\pi_L$ and of the base points $P_1, \ldots, P_n$ by $\pi_R$ differ by an
inversion or a similarity.
		\item
			There exists $m \leq n$ such that $p_1, \ldots, p_m$ are collinear and
$P_{m+1}, \ldots, P_n$ are collinear, up to permutation of indices.
	\end{enumerate}
\end{cor}
\begin{proof}
	Since by hypothesis $K_{\Pi} \cap \SEC$ has dimension $\geq 1$, it follows
that $B_{\Pi}$ is not empty (see Remark \ref{remark:mobility_bonds}). Hence
there is at least one inversion/similarity/ collinearity/butterfly bond, and
then the result follows from Corollaries (\ref{cor:inversion-similarity}),
(\ref{cor:butterfly}) and (\ref{cor:collinear}).
\end{proof}

\begin{remark}
	As the bonds are determined by the invariant linkage parameters, they are
independent of the leg lengths. As a consequence a hexapod, which has $40$
solutions for the direct kinematics over $\mathbb C$ (see\ Remark
\ref{rem:cardinality}), is free of bonds and therefore also free of 
self-motions. 
	Due to the fact that condition (i) of Corollary \ref{corollary:mob1} is not
sufficient, the converse is not true; i.e.\ there exist hexapods with less
than $40$ solutions for the direct kinematics, which are free of self-motions
(e.g.\ hexapods where the platform and base are planar and projective --- but
not affine --- equivalent \cite{nawratilproj}).
\end{remark}

We conclude stating our last result, concerning constraints on base and
platform points of $n$--pods with higher mobility.

\begin{theorem}
\label{theorem:mobility_two}
	Let $\Pi$ be an $n$--pod with mobility $2$ or higher. Then one of the
following holds:
	\begin{itemize}
		\item[(a)]
			there are infinitely many pair $(L,R)$ of elements of $S^2$ such that the
points $\pi_L(p_1), \ldots, \pi_L(p_n)$ and $\pi_R(P_1), \ldots, \pi_R(P_n)$
differ by an inversion or a similarity;
		\item[(b)]
			there exists $m \leq n$ such that $p_1, \ldots, p_m$ are collinear and 
$P_{m+1} = \ldots = P_n$, up to permutation of indices and interchange between
base and platform;
		\item[(c)]
			there exists $m \leq n$ with $1 < m < n-1$  such that $p_1, \ldots, p_m$
lie on a line $g \subseteq \R^3$ and $p_{m+1}, \ldots, p_n$ lie on a line $g'
\subseteq \R^3$ parallel to $g$, and $P_1, \ldots, P_m$ lie on a line $G
\subseteq \R^3$ and $P_{m+1}, \ldots, P_n$ lie on a line $G' \subseteq \R^3$
parallel to $G$, up to permutation of indices.
	\end{itemize}
\end{theorem}
\begin{proof}
	Since $\Pi$ has mobility at least $2$, it has infinitely many bonds (see
Remark \ref{remark:mobility_bonds}). Assume that $\Pi$ admits one collinearity
bond, then we have b) with $m = n$. Assume that it admits infinitely many
butterfly points, then in particular by Corollary \ref{cor:butterfly} there
exists $m \leq n$ such that $p_1, \ldots, p_m$ are collinear and $P_{m+1},
\ldots, P_n$ lie on infinitely many different lines, and therefore we have (b).
Hence we are left with the case when we have infinitely many inversion or
similarity bonds. If these bonds provide infinitely many different left and
right vectors, we are in case (a). Otherwise we have that there are infinitely
many inversion or similarity points with the same left and right vectors
$(L,R)$. We want to argue that in this case both sets $\mcal{U} = \big\{
\pi_L(p_1), \ldots, \pi_L(p_n) \big\}$ and $\mcal{V} = \big\{ \pi_R(P_1),
\ldots, \pi_R(P_n) \big\}$ consist of two points. In fact by Corollary
\ref{cor:inversion-similarity} the inversion/similarity associated to
these bonds maps $\pi_L(p_i)$ to $\pi_R(P_i)$, so $\mcal{U}$ and $\mcal{V}$
have the same cardinality; on the other hand any inversion or similarity is
completely specified if we prescribe the image of three points, so if the
cardinality of $\mcal{U}$ were greater than $2$ then we would have only one
inversion or similarity. Moreover we can exclude the case when both $\mcal{U}$
and $\mcal{V}$ are given by one point, since this falls in case (b). Hence $p_1,
\ldots, p_n$ are arranged on two parallel lines, and the same holds for $P_1,
\ldots, P_n$. From this and the fact that the inversions/similarities should
map $\pi_L(p_i)$ to $\pi_R(P_i)$ it follows that the only possible
configurations are the ones described in (c). As a side remark, since two
points fix a similarity it follows that in this case we have just one
similarity point and infinitely many inversion points.
\end{proof}

As already mentioned in the Introduction, we can also formulate some geometric
conditions on base and platform points in case (a) of Theorem
\ref{theorem:mobility_two}. This can be done by a new technique, called
\emph{M\"obius Photogrammetry}, which is developed by the authors in
\cite{GalletNawratilSchicho}. Moreover it should be noted that based on the
results obtained with this method a complete classification of pentapods with
mobility $2$ was achieved in \cite{NawratilSchicho1, NawratilSchicho2}.

\section*{Acknowledgments}

The first and third author's research is supported by the Austrian Science Fund (FWF): W1214-N15/DK9 and P26607 - ``Algebraic Methods in Kinematics: Motion Factorisation and Bond Theory''. The second author's research is funded by the Austrian Science Fund (FWF): P24927-N25 - ``Stewart Gough platforms with self-motions''.                                                                      

\bibliographystyle{plain}

\begin{thebibliography}{10}

\bibitem{bottroth}
Oene Bottema and Bernhard Roth.
\newblock {\em Theoretical {K}inematics}.
\newblock Applied Mathematics and Mechanics. North-Holland Publishing Company,
  Amsterdam, 1979.

\bibitem{bricardoct}
Raoul Bricard.
\newblock M\'{e}moire sur la th\'{e}orie de l'octa\`{e}dre articul\'{e}.
\newblock {\em Journal de Math\'{e}matiques pures et appliqu\'{e}es,
  Liouville}, 3:113--148, 1897.

\bibitem{bricard}
Raoul Bricard.
\newblock M\'{e}moire sur les d\'{e}placements \`{a} trajectoires
  sph\'{e}riques.
\newblock {\em Journal de \'{E}cole Polytechnique(2)}, 11:1--96, 1906.

\bibitem{GalletNawratilSchicho}
Matteo Gallet, Georg Nawratil, and Josef Schicho.
\newblock M{\"o}bius photogrammetry.
\newblock Submitted.

\bibitem{GeissSchreyer}
Florian Gei{\ss} and Frank-Olaf Schreyer.
\newblock A family of exceptional {S}tewart-{G}ough mechanisms of genus 7.
\newblock In {\em Interactions of classical and numerical algebraic geometry},
  volume 496 of {\em Contemp. Math.}, pages 221--234. Amer. Math. Soc.,
  Providence, RI, 2009.

\bibitem{HegeduesSchichoSchroecker}
G{\'a}bor Heged{\"u}s, Josef Schicho, and Hans-Peter Schr\"ocker.
\newblock The {T}heory of {B}onds: {A} {N}ew {M}ethod for the {A}nalysis of
  {L}inkages.
\newblock {\em Mechanism and Machine Theory}, 70:404--424, 2013.

\bibitem{Karger}
Adolf Karger.
\newblock Self-motions of {S}tewart-{G}ough platforms.
\newblock {\em Comput. Aided Geom. Design}, 25(9):775--783, 2008.

\bibitem{KargerHusty}
Adolf Karger and Manfred Husty.
\newblock Classification of all self-motions of the original {S}tewart-{G}ough
  platform.
\newblock {\em Computer-Aided Design}, 30(3):205 -- 215, 1998.

\bibitem{Merlet89}
Jean-Pierre Merlet.
\newblock Singular {C}onfigurations of {P}arallel {M}anipulators and
  {G}rassmann geometry.
\newblock {\em I. J. Robotic Res.}, 8(5):45--56, 1989.

\bibitem{Nawratil2014}
Georg Nawratil.
\newblock Introducing the theory of bonds for {S}tewart {G}ough platforms with
  self-motions.
\newblock {\em ASME Journal of Mechanisms and Robotics}, 6(1):011004, 2014.

\bibitem{nawratilproj}
Georg Nawratil.
\newblock Non-existence of planar projective {S}tewart {G}ough platforms with
  elliptic self-motions.
\newblock In {\em Computational Kinematics ({B}arcelona, 2013)}, pages 49--57.
  Springer, Dordrecht, 2014.

\bibitem{NawratilSchicho1}
Georg Nawratil and Josef Schicho.
\newblock Pentapods with {M}obility 2.
\newblock Submitted [arXiv:1406.0647].

\bibitem{NawratilSchicho2}
Georg Nawratil and Josef Schicho.
\newblock Self-motions of pentapods with linear platform.
\newblock Submitted [arXiv:1407.6126].

\bibitem{RongaVust}
Felice Ronga and Thierry Vust.
\newblock Stewart platforms without computer?
\newblock In {\em Real analytic and algebraic geometry ({T}rento, 1992)}, pages
  197--212. de Gruyter, Berlin, 1995.

\end{thebibliography}

\end{document}